\documentclass[smallextended]{svjour3}

\usepackage{amssymb}
\usepackage{fullpage}
\usepackage{times}
\usepackage{algorithm}
\usepackage[noend]{algpseudocode}
\usepackage{lineno}
\usepackage{amsmath}
\usepackage{ntheorem}
\usepackage{graphicx}
\usepackage{tikz}
\usepackage{pgfplots}
\usepackage{todonotes}
\usetikzlibrary{
                chains,
                positioning,
                shapes.geometric,
                shapes,arrows,
                patterns,
                decorations.markings
                }
\usepackage{subcaption}
\usepackage{hyperref}
\usepackage{float}
\usepackage{flafter}

\makeatletter
\def\BState{\State\hskip-\ALG@thistlm}
\makeatother

\smartqed
\begin{document}

\tikzstyle{block} = [rectangle, draw, fill=blue!20, 
    text width=5em, text centered, rounded corners, minimum height=3em]
\tikzstyle{line} = [draw, -latex']

\title{RELF: \underline{R}obust Regression 
Extended with \underline{E}nsemble \underline{L}oss \underline{F}unction}

\author{Hamideh Hajiabadi \and Reza Monsefi \and Hadi Sadoghi Yazdi }

\institute{Hamideh Hajiabadi \at
              Computer Department, Ferdowsi University of Mashhad (FUM) \\
 \email{Hamideh.hajiabadi@mail.um.ac.ir}          
           \and
           Reza Monsefi \at
              Computer Department, Ferdowsi University of Mashhad (FUM) \\
              \email{monsefi@um.ac.ir}  \and
           Hadi Sadoghi Yadi \at
              Computer Department, Ferdowsi University of Mashhad (FUM) \\
              \email{h-sadoghi@um.ac.ir}
}

\date{Received: date / Accepted: date}
\maketitle

\begin{abstract}
Ensemble techniques are powerful approaches that combine several weak learners to build a stronger one. As a meta-learning framework, ensemble techniques can easily be applied to many machine learning methods. Inspired by ensemble techniques, in this paper we propose an ensemble loss functions applied to a simple regressor. We then propose a half-quadratic learning algorithm in order to find the parameter of the regressor and the optimal weights associated with each loss function. Moreover, we show that our proposed loss function is robust in noisy environments. For a particular class of loss functions, we show that our proposed ensemble loss function is Bayes consistent and robust. Experimental evaluations on several data sets demonstrate that the our proposed ensemble loss function significantly improves the performance of a simple regressor in comparison with state-of-the-art methods. 
\end{abstract}


\section{Introduction}
\label{sec:introduction}
Loss functions are fundamental components of machine learning systems and are used to train the parameters of the learner model. Since standard training methods aim to determine the parameters that minimize the average value of the loss given an annotated training set, loss functions are crucial for  successful trainings  \cite{xiao2017ramp,zhao2010convex}. Bayesian estimators are obtained by minimizing the expected loss function. Different loss functions lead to different Optimum Bayes  with possibly different characteristics. Thus, in each environment the choice of the underlying loss function is important, as it will impact the performance\cite{uhlich2012bayesian,wang2003multiscale}.

\medskip

Letting $\vec{\hat{\theta}}$ denote the estimated parameter of a correct parameter $\vec{\theta} $, the loss function $L(\vec{\hat{\theta}},\vec{\theta})$ is a positive function which assigns a loss value to each estimation, indicating how inaccurate the estimation is \cite{steinwart2008support}. Loss functions assign a value to each sample, indicating how much that sample contributes to solving the optimization problem. 
Each loss function comes with its own advantages and disadvantages. In order to put our results in context, we start by reviewing three popular loss functions ($0$-$1$, Ramp and Sigmoid) and we will give an overview of their advantages and disadvantages. 

\medskip

Loss functions assign a value to each sample representing how much that sample contributes to solving the optimization problem. If an outlier is given a very large value by the loss function, it might dramatically affect the decision function \cite{hajiabadi2017extending}. The $0$-$1$  loss function is known as a robust loss because it assigns value 1 to all misclassified samples --- including outliers --- and thus an outlier does not influence the decision function, leading to a robust learner. On the other hand, the $0$-$1$ loss penalizes all misclassified samples equally with value $1$, and since it does not enhance the margin, it cannot be an appropriate choice for applications with margin importance \cite{xiao2017ramp}.

\medskip

The Ramp loss function, as another type of loss functions, is defined similarly to the $0$-$1$ loss function with the only difference that ramp loss functions also penalize some correct samples, those with small margins. This minor difference makes the Ramp loss function appropriate for applications with margin importance \cite{tang2018ramp,xiao2017ramp}. On the downside, the Ramp loss function is not differentiable, and hence not suitable for optimization purposes. 

\medskip

The Sigmoid loss function is almost the same as $0$-$1$ loss functions, except that it is differentiable, which in turn  makes optimization significantly easier. However, it assigns a (very small) non-zero value to correct samples, meaning that those samples would also contribute to solving the optimization problem. Hence, in spite of easier optimization, the Sigmoid loss function leads to a less sparse optimization problem in comparison with the $0$-$1$ loss function \cite{vapnik1998statistical}. 

\medskip

\par
There are many other examples of different loss functions showing that while a loss function might be good for certain applications,  it might be unsuitable for many others. Inspired by ensemble methods, in this paper we propose the use of an ensemble of loss functions during the training stage. The ensemble technique is one of the most influential learning approaches. Theoretically, it can boost weak learners, whose accuracies are slightly better than random guesses, into arbitrarily accurate strong learners \cite{napoles2017rough}. This method would be effective when it is difficult to design a powerful learning algorithm directly \cite{bai2014bayesian,zhang2016bayesian,mannor2001weak}. As a meta-learning framework, it can be applied to almost all machine learning algorithms to improve their prediction accuracies. 

\medskip
Our goal in this paper is to propose a new ensemble loss function, which we later apply to a simple regressor. Half-Quadratic (HQ) minimization, which is a fast alternating direction method, is used to learn regressor's parameters. In each iteration, the HQ tries to approximate the convex or non-convex cost function with a convex one and pursue optimization \cite{geman1992constrained}. Our main contributions are as follows. 
\begin{itemize}
\item Inspired by ensemble-induced methods, we propose an ensemble loss whose properties are inherited from its base loss functions. Moreover, we show that each loss is a special case of our proposed loss function.
\item We develop both online and offline learning frameworks to find the weights associated with each loss function, and so to build an ensemble loss function. For a particular class of base losses, we prove that the resulting ensemble loss function is Bayes consistent and robust. 
\end{itemize}
\par
This paper is structured as follows. We review some existing loss functions and several promising ensemble regressors in Section~\ref{sec:related works} which contains two subsections for each. We briefly explain about Half-Quadratic (HQ) programming in Section~\ref{sec: HQ}. Our proposed framework is discussed in Section~\ref{se: proposed method}, for which we provide implementation and test results in Section~\ref{sec: experiments}.  Finally, we conclude in Section~\ref{sec: conclusion} with a list of problems for future work.

\section{A Review of Loss Functions}\label{sec:related works}

This work draws on two broad areas of research: loss functions, and ensemble-based regression methods. In this section, these two areas are fully covered. 

\subsection{\textbf{A Review of Loss Functions}}
In machine learning, loss functions are divided in two categories, margin-based and distance-based \cite{steinwart2008support}. Margin-based loss functions  are used for classification \cite{feng2016robust,zhang2001text,khan2013semi,bartlett2006convexity}, while Distance-based loss functions are generally used for regression. In this paper, we only focus on distance-based loss functions.  

\paragraph{Distance-Based Loss Functions} Let $\vec{x},y, f(\vec{x})$ denote an input, the corresponding true label and the estimated label respectively. A distance-based loss function, is a penalty function $\phi(y-f(\vec{x}))$ where $y-f(\vec{x})$ is called distance \cite{sangari2016convergence,chen2017kernel}. The risk associated with the loss function $\phi(.)$ described as

$$
R_{\phi,P}(f)=\int_{X\times Y}\phi (y-f(\vec{x}))dP(X,Y)
$$
\noindent where $P(X,Y)$ is the joint probability distribution over $X$ and $Y$. The ultimate aim of a learning algorithm is to find a function $f^*$ among a fixed class of function $\mathcal{F}$ for which the risk $R_{\phi,P}(f^{*})$ is minimal \cite{steinwart2008support,painsky2016isotonic},
    $$f^{*} =\arg\min_{f \in \mathcal{F}} R_{\phi,P}(f).$$  
Generally $R(f)$ cannot be computed because the distribution $P(X,Y)$ is unknown. However, an approximation of $R(f)$ which is called empirical risk can be computed by averaging the loss function on the training set containing $n$ samples \cite{holland2016minimum},

$$
R_{\text{emp}}(f)=\frac {1}{n}\sum _{i=1}^{n}\phi(y_{i}-f(\vec{x_{i}})) \text{, and }$$
 $$ \hat{f}  = \arg \min _{f\in \mathcal{F}} R_{\text{emp}} (f).  $$

\medskip
A loss function should be such that we arrive at Bayes decision function after minimizing the associated risk under given loss function. The loss is said to be Bayes consistent if by increasing the samples size, the resulting function converges to the Bayes decision function \cite{zhang2004statistical,buja2005loss,friedman2000additive}. The Bayes decision function is fully explained in the cited papers. 

\begin{table}[h!]
\caption{\label{table:Bayse loss functions} Well-known Bayes Consistent Loss Functions  }
\centering
\begin{tabular}{l  l  } 
\hline
Algorithm 	& $\phi(y,f(\vec{x}))  $\\
\hline
Square \cite{lopez2018robust}&	$(1-yf(\vec{x}))^2$\\
Hinge \cite{bartlett2008classification} &	$\max(0,1-yf(\vec{x}))$\\
Exp. 	&$ \exp (-\beta y f(\vec{x}))$\\
Logistic \cite{fan2008liblinear}&	$\ln (1+e^{-(yf(\vec{x})} )$\\
\hline
\end{tabular}
\end{table}

\par
Table \ref{table:Bayse loss functions} and Fig.\ref{loss} illustrate known examples of Bayes consistent loss functions \cite{masnadi2010design,masnadi2009design}. Fig.\ref{loss} shows that all of these functions are convex and unbounded. As we mentioned in the previous section, loss functions assign a value to each sample which indicates how much that sample contributes to solving the optimization problem. Unbounded loss functions assign large values to samples with large errors and thus they are more sensitive to noise. Hence, under unbounded loss functions, the robustness deteriorates in noisy environments. 

\begin{figure}[h]
\centering
\begin{tikzpicture}
 \begin{axis}[grid,xlabel={$y-f(x)$},ylabel=loss]

\addplot [color=red,thick, domain=-2:2.5]{ln(1+exp(-x))};
\addplot [color=blue,thick, domain=-2:2.5] plot {max(0,1-x)} ;
\addplot [color=green,thick, domain=-2:2] plot {x^2} ;
\addplot [color=orange,thick, domain=-1.5:2.5] plot {exp(-x)} ;
 \end{axis}
\end{tikzpicture}
\caption{\label{loss} Well-known Bayes consistent loss functions, Hinge (blue), Square (green), Logistic (red) Exp. (orange)}
\end{figure}
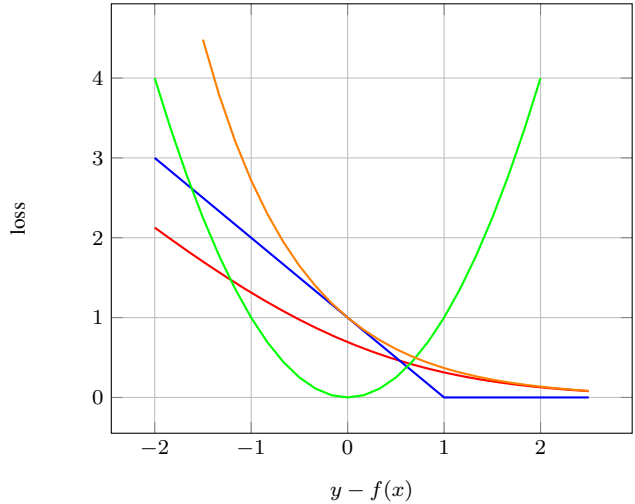
\par
Fig. \ref{savage} shows two unbounded loss functions (the Exp. loss and the Logistic loss) and a bounded one (the Savage loss).  SavageBoost which uses the Savage loss function leads to a more robust learner in comparison with AdaBoost and Logitboost which uses the Exp. loss and the Logistic loss function respectively \cite{masnadi2009design}. Several researchers suggested that although convex loss functions make optimization easier, the robustness deteriorates in the presence of outliers \cite{miao2016rboost}. For example, while LS-SVR uses the Square loss and is sensitive to outliers, RLS-SVR uses the non-convex least squares loss function to overcome the limitation of LS-SVR \cite{wang2014robust}. 

\medskip
There are many other distance-based loss functions which are not considered as Bayes consistent but have been widely used in literature. Some of these are shown in Table \ref{table:widely used loss}. 

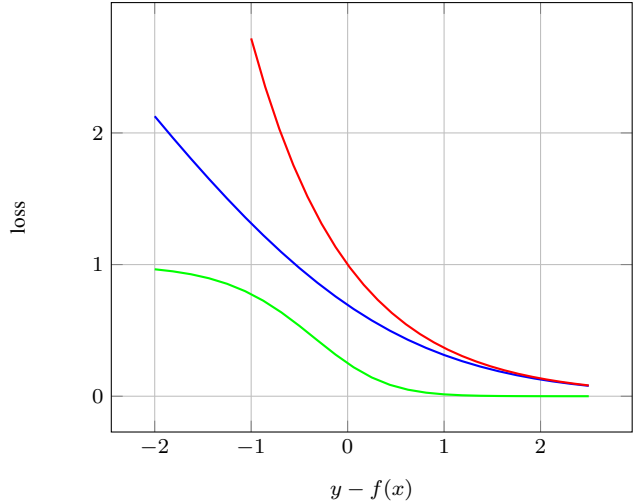
\begin{figure}[h]
\centering
\begin{tikzpicture}
 \begin{axis}[grid,xlabel={$y-f(x)$},ylabel=loss]

\addplot [color=blue,thick, domain=-2:2.5]{ln(1+exp(-x))};
\addplot [color=red,thick, domain=-1:2.5] plot {exp(-x)};
\addplot [color=green,thick, domain=-2:2.5] plot {1/(1+exp(2*x))^2};

 \end{axis}
\end{tikzpicture}
\caption{\label{savage} the Exp. loss (red), the Logistic loss (blue) and the Savage loss function (green)}
\end{figure}

\begin{table}[h]
\caption{\label{table:widely used loss} Some widely used loss functions}
\centering
\begin{tabular}{l l }
\hline
Name&	Description\\
\hline
Absolute \cite{vapnik2013nature}&
$l_{abs}=|y-f(\vec{x})|$\\
Huber \cite{huber1964robust}&
$l_{Huber}=\begin{cases}
\frac{(y-f(\vec{x}))^2}{4\epsilon}&\text{for }|y-f(\vec{x})|<2 \epsilon \\
|y-f(\vec{x})|-\epsilon &\text{otherwise}  
\end{cases}$\\

C-loss \cite{liu2006correntropy} &
$l_{corr}=1- \exp ({\frac{-(y-f(\vec{x}))^2}{2\sigma^2}})$\\

$\epsilon$-insensitive Ramp-loss \cite{tang2018ramp} &
$l_{\epsilon-Ramp}=\begin{cases} \alpha & \text{for }|y-f(\vec{x})|\geq \alpha  \\
|y-f(\vec{x})| &\text{for }  \epsilon<|y-f(\vec{x})|<\alpha\\
0 &\text{for }  |y-f(\vec{x})|\leq \epsilon 
\end{cases}$\\
\hline
\end{tabular}
\end{table}

\par
\smallskip
The Huber loss function which is the combination of the Square and the Absolute function is shown in Table \ref{table:widely used loss}. It was utilized in robust regression, and the results showed a significant improvement in its robustness comparing with the standard regression. The only difference between Absolute and Huber loss functions is at the point of $y-f(\vec{x})=0$. While the Absolute loss function is not differentiable at the point of $y-f(\vec{x})=0$, this problem is fully addressed by the Huber loss function \cite{huber1964robust}. 

\par
\smallskip
Correntropy which is rooted from the Renyi's entropy is a local similarity measure. It is based on the probability of how much two random variables are similar in the neighbourhood of the joint space. The joint space can be adjusted by a kernel bandwidth. The C-loss function which is shown in Table \ref{table:widely used loss} is
inspired by Correntropy criteria and is known as a robust loss function \cite{liu2006correntropy}. Several researchers have used the C-loss function to improve the robustness of their learning algorithms \cite{peng2017maximum,liu2007correntropy,zhao2012adaptive,chen2014steady,chen2016generalized}. 

\par
\smallskip
$\epsilon$-insensitive Ramp-loss which is inspired by Ramp loss function \cite{xiao2017ramp} is proposed in \cite{tang2018ramp}. It is a kind of robust and margin enhancing loss function which was applied to a linear and kernel Support Vector Regression (SVR) in \cite{tang2018ramp}. The weights assigned to each sample are limited to be no higher than a pre-defined value of $\alpha$, thus the
negative effect brought by outliers can be effectively reduced. 

\smallskip
Loss functions can also be used in Dimensional Reduction (DR) purposes which is a kind of feature reduction technique. For example, in \cite{xie2018matrix} nuclear norm (N norm) is used as the loss function in solving the DR problem based on matrix regression
model. N norm can
well preserve the low-rank information of samples and result in a low dimensional data and would be a good choice for DR purposes.

\par
\smallskip
While each individual loss function has its own advantages and disadvantages, in this paper we propose an ensemble loss function which is a combination of several individual loss functions. By doing so, we hope to produce a strong ensemble loss function which its advantages are inherited from each individual loss function. In the next section Ensemble learning and several promising ensemble method are discussed.

\subsection{\textbf{Ensemble Learning}}
Ensemble learning combines outputs of several models to make a prediction. They aim to improve the overall accuracy and robustness  over individual models \cite{dudek2016pattern,mendes2012ensemble}. 
The focus of the most ensemble methods is on classification problems, however relatively few have paid attention to regression tasks. Ensemble methods are comprised of two steps: (1) generation step in which a set of base learners are built (2) integration step which involves in combining these base learners in order to make the final prediction. Base learners can be combined statistically or dynamically. A static combination uses predefined combination rules for all instances while the dynamic one makes use of different rules for different instances \cite{ko2008dynamic,cruz2015meta}. In the following the most promising Ensemble methods are discussed. 
\par
\smallskip
Two of the most appealing ensemble methods for regression trees are Bagging and random forest. They are the most commonly used algorithms due to their simplicity, consistency and accuracy \cite{breiman1996bagging,breiman2001random}. Breiman Random Forest approach first constructs a multitude of decision trees and output the class that is the mode of the classes (classification) or mean prediction (regression) of the individual trees. The trees are later modified to incorporate randomness, a split used at each node to select randomly feature subset. moreover, the subset considered in one node is completely independent of the subset in the previous node. 
\par
\smallskip
There are some interesting ensemble methods on neural networks,  one which is based on negative correlation, is called Evolutionary Ensembles with Negative Correlation Learning (EENCL) \cite{liu2000evolutionary}.  it randomly changes the weights of an existing neural network by using mutation. It also obtains ensemble size automatically. The importance of this approach is due to its theoretical foundations. Another interesting method on neural networks was presented in 2003 \cite{islam2003high}
which builds the base learners and the ensemble one simultaneously. Therefore, it saves time in comparing with the methods that first generate base learners and then combine them. 
\par
 \smallskip
There are some other research based on local experts which are specialized in local prediction, aim to achieve better prediction accuracy in comparison with globally best models. Recently a Locally Linear Ensemble Regressor (LLER) has been proposed which first divides the original dataset into some locally linear regions by employing an EM procedure and then builds a linear model per each region in order to form the ensemble model \cite{kang2018locally}.    
\par
\smallskip
In this paper, we dynamically combine several loss functions and the weights associated with each individual loss is obtained through the training phase using Half-Quadratic (HQ) optimization. In the following section, we provide an overview of HQ optimization.

\section{Half-Quadratic Optimization}\label{sec: HQ}
Let $\phi_\nu (.)$ be a function of vector $\vec{\nu}\in R^n$ and defined as $\phi(\vec{\nu})=\sum_{j=1}^n \phi_{\vec{\nu}}(\nu_i)$ where $\nu_i$ is the $j$th entry of $\vec{\nu}$. In machine learning problems, we often aim to minimize an optimization problem like 

\begin{equation}\label{eq:optimization}
min\sum_{j=1}^n\phi(\nu_j )+J(\vec{\nu})
\end{equation}

\noindent
where usually $\vec{\nu}$ is the learner's parameters and $\phi(.)$ is a loss function which can be convex or non-convex and $J(\vec{\nu})$ is a convex penalty function on $\vec{\nu}$ which is optional and considered as the regularization term. According to Half-Quadratic (HQ) 
optimization \cite{he2014half}, for a fixed $\nu_j$ , we have

\begin{equation}\label{eq: hq2}
    \phi(\nu_j)= \min_{p_j}
Q(\nu_j,p_j ) + \psi(p_j )
\end{equation}
where $\psi(.)$ is the convex conjugate function of $\phi(.)$ and $Q(\nu_j,p_j )$ is an HQ function which is modeled by the additive or the multiplicative form. The additive and the multiplicative forms for $Q(\nu_j,p_j)$  are respectively formulated as $(\nu_j-p_j )^2$ and $\frac{1}{2} p_j \nu_j^2$. Let $Q_{\vec{\nu}} (\vec{\nu}, \vec{p})=
\sum_{j=1}^n Q(\nu_j, p_j)$, the vector form of Equation \ref{eq: hq2} is as follow

\begin{equation}\label{eq:hq3}
\phi_{\vec{\nu}} (\vec{\nu})=\min_{\vec{p}}Q_{\vec{\nu}}(\vec{\nu},\vec{p})+\sum_{j=1}^n\psi(p_j ).
\end{equation}

\medskip
\noindent
By substituting the equations \ref{eq:hq3} for $\phi(\nu_j)$ in Equation \ref{eq:optimization}, the following cost function is obtained,

\begin{equation}\label{eq:HQ}
\min_{\vec{\nu}}\lbrace \phi_{\vec{\nu}} (\vec{\nu})+J(\vec{\nu})\rbrace=\min_{\vec{p},\vec{\nu}}Q_{\vec{\nu}}(\vec{\nu},\vec{p})+\sum_{j=1}^n\psi(p_j )+J(\vec{\nu}).
\end{equation}

\noindent
Assuming the variable $\vec{\nu}$ fixed, the optimal value for $\vec{p}$ is obtained by the minimization function $\delta(.)$ which is computed as
$$p_j=\delta(.)=\arg\min_{\vec{p}_j} \left\lbrace Q(\vec{\nu},\vec{p})+\sum_{j=1}^n \psi(p_j)\right\rbrace$$

\noindent 
and it is only related to $\phi(.)$ (some specific forms of  $\phi(.)$ and the corresponding $\delta(.)$ is shown in Fig.\ref{HQ}). For each $\vec{\nu}$ the value of $\delta(.)$ is such that
\begin{equation*}
Q(\nu_j,\delta(\nu_j))+\psi(\delta(\nu_j)) \leq Q(\nu_j,p_j)+\psi(p_j).
\end{equation*}
\par
The optimization problem \ref{eq:HQ} is convex since it is the summation of three convex functions and is can be minimized alternating steps as follows 

\begin{equation*}
\vec{p}^{(k+1)}=\delta(\vec{\nu}^{(k)})
\end{equation*}
\begin{equation*}
\vec{\nu}^{(k+1)}=\arg\min_{\vec{\nu}}\left\lbrace Q(\vec{\nu},\vec{p}^{(k+1)})+J(\vec{\nu})\right\rbrace
\end{equation*}

\noindent
where $k$ denotes the iteration number. At each iteration, the objective function is decreased until convergence. The HQ method is fully explained in \cite{geman1995nonlinear,he2014half}.

\begin{figure}[h]
\centering
\includegraphics[width=\textwidth]{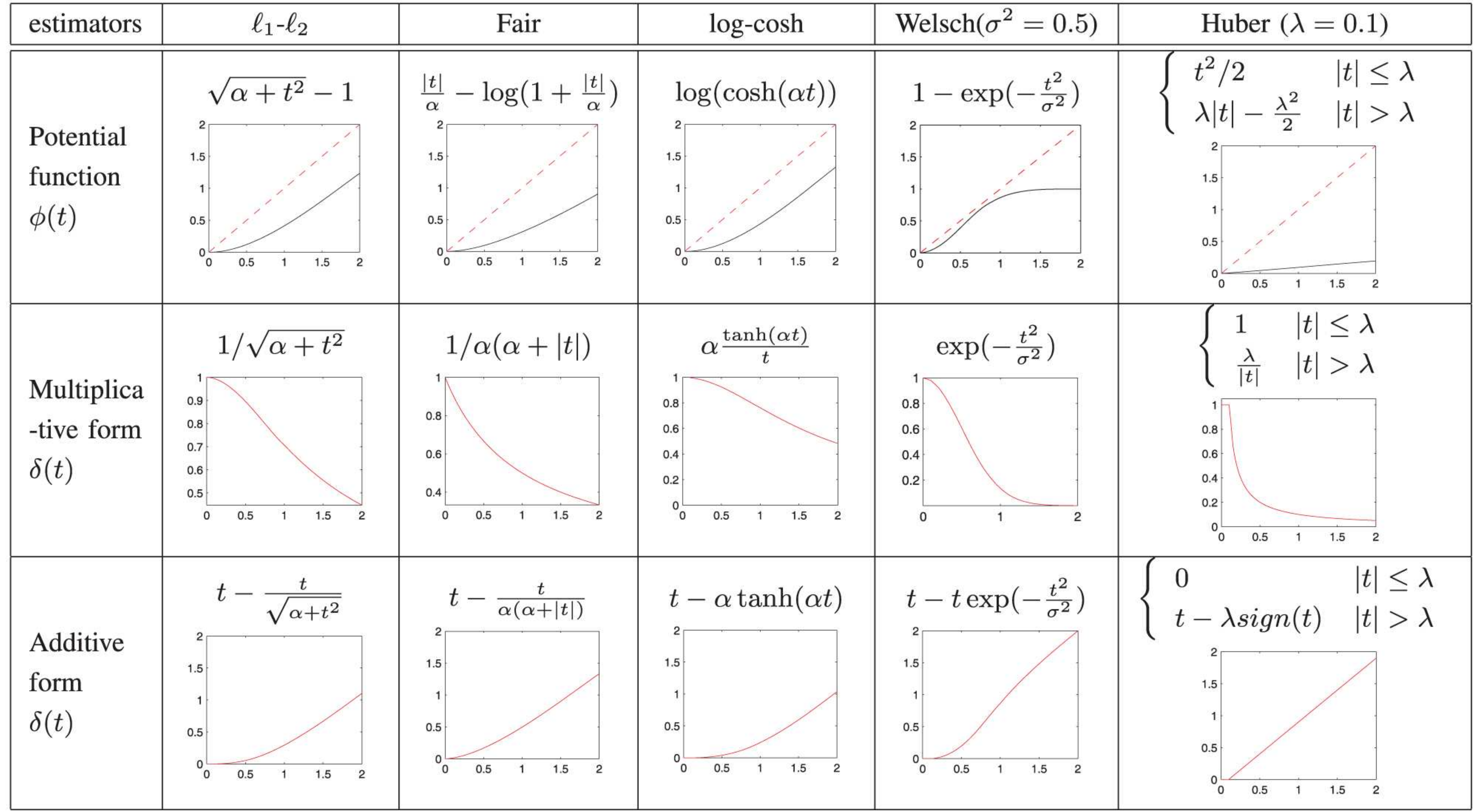}
\caption{\label{HQ} some estimators and the equivalent $\delta(.)$ for additive and multiplicative HQ forms \cite{he2014half}}.
\end{figure}

\section{The Proposed Method}\label{se: proposed method}

Let $(X,Y)$ denote all samples where $X=\lbrace \vec{x_1},\vec{x_2},\ldots ,\vec{x_N}\rbrace$ are the inputs and $Y=\lbrace y_1, y_{2}, \ldots ,y_{N}\rbrace$ are the labels. Let $w$ be the parameters of the regressor which estimates the predicted label by $\hat{y}=f(\vec{x},\vec{w})$. Let $\{\phi_i(y,\hat{y})\}_{i=1}^m$ denote $m$ weak loss functions. We aim to find an optimal $\vec{w}$ and the best weights, $\{\lambda_1,\lambda_2,\dots,\lambda_m\}$, associated with each loss function. We need to add a further constraint to avoid yielding near zero values for all $\lambda_i$ weights. Our proposed ensemble loss function is defined in Equation \ref{eq:Proposed loss}. 
\begin{equation}\label{eq:Proposed loss}
L=\sum_{k=1}^m \lambda_k \phi_k(y,\hat{y}),\sum_{k=1}^m \lambda_k =1,\lambda _{k}\geq 0
\end{equation}

\noindent
Fig. \ref{fig: learning-diagram} shows the model of our proposed 
method with the bold box representing the novelty of this paper. In 
the training phase, the weights associated with $M$ loss functions are learned and our proposed ensemble loss function is formed. To ease computations we 
select each $\phi _{i}(y-f(\vec{x}))$ from M-estimator functions introduced in 
Fig.\ref{HQ}. 

\tikzset{
  big arrow/.style={black,line width=1mm,
    decoration={markings,mark=at position 1 with {\arrow[scale=1,#1]{>}}},
    postaction={decorate},
    shorten >=2pt},
  big arrow/.default=black}
\begin{figure}[h]
\centering
\begin{tikzpicture}

\node  [block,fill=red!30,minimum size=1.5cm,fill=magenta] at (8,6) (risk){Risk Optimizer};

\node[cylinder,draw=black,thick,aspect=0.4,minimum height=1.7cm,minimum width=1.5cm,shape border rotate=90,cylinder uses custom fill, cylinder body fill=magenta,cylinder end fill=magenta!10,align=center,node distance=2.2cm,above of=risk]   (loss set){Set of \\ Loss Functions};

\node  [block,minimum size=2cm,fill=cyan] at (3,7.5) (regressor){Regressor};

 \node [rectangle, draw=black, text width=4.5cm, minimum height=6.5cm,thick] (cc) at (8,7.5) {};
 
 \node [rectangle, draw=black, text width=14cm, minimum height=8cm,dashed] (training phase) at (4,7.5) {};
 
 \node [rectangle, draw=white, text width=3cm, minimum height=1cm,dashed,node distance=6.7cm, left of=training phase, rotate= 90] (label1)  {Training Phase};

  \node[cylinder,draw=black,thick,aspect=0.4,minimum height=1.7cm,minimum width=1.5cm,shape border rotate=90,cylinder uses custom fill, cylinder body fill=cyan,cylinder end fill=cyan!40,align=center] (train) at (-1,7.5) (train){Training \\ Samples};

 \node[cylinder,draw=black,thick,aspect=0.4,minimum height=1.7cm,minimum width=1.5cm,shape border rotate=90,cylinder uses custom fill, cylinder body fill=cyan,cylinder end fill=cyan!40,align=center,node distance=6cm,below of=train] (test)  {Test \\ Samples};
 
\node  [block,minimum size=2cm,fill=cyan,node distance=4cm, right of = test]  (relf){RELF}; 

\node [rectangle, draw=black, text width=14cm, minimum height=4cm,dashed] (test phase) at (4,1) {};
 
 \node [rectangle, draw=white, text width=3cm, minimum height=1cm,dashed,node distance=6.7cm, left of=test phase, rotate= 90] (label2)  {Testing Phase};

 \draw[->] (train) -- (regressor);
 \draw[->] (test) -- (relf);
\draw[->] (regressor) -- (cc);
\draw[->] (cc) -- (regressor);
\draw[big arrow] (cc) -- node[above]{Parameters}(relf);

\end{tikzpicture}
\caption{\label{fig: learning-diagram} The Proposed Model.}
\end{figure}
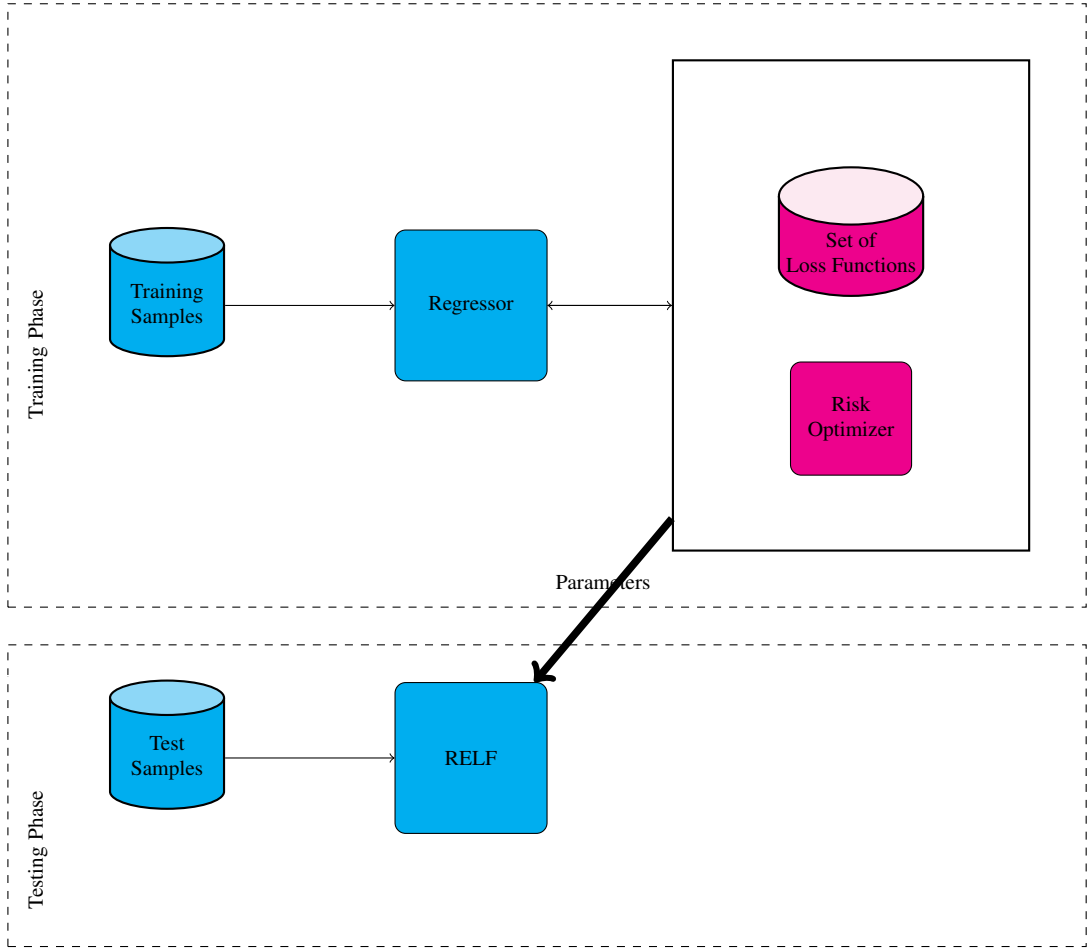
\par
The expected risk of the proposed loss 
function is defined as 

\begin{equation}\label{eq:risk}
\begin{aligned}
R(f)=E_X \left\lbrack \sum_{k=1}^{m}\lambda _{k}\phi _{k}(y,\hat{y})\right\rbrack
=& \sum_{k=1}^{m}\lambda _{k}E\lbrack \phi _{k}(y,f(\vec{x}))\rbrack\\
s.t. &  \sum_{k=1}^{m}{\lambda _{k}}=1,\lambda _{k}\geq 0.
\end{aligned}
\end{equation}

\noindent
As known, under the Square loss, the Bayes estimator is the posterior mean and 
the Bayes estimator with respect to the Absolute loss is the posterior median. As shown in Equation \ref{eq:risk}, the Bayes estimator of our proposed loss function is a weighted summation of all Bayes estimators associated with each individual loss function. For example, if we ensemble the Square and the Huber loss functions, then, the Bayes estimator of our ensemble loss function is a trade-off between the posterior mean and the median. 

\smallskip
In the next two subsections, two properties of 
our proposed loss function are discussed. The third subsection contains the training scheme. 

\subsection{Bayes Consistency}

The advantages 
of using a Bayes-consistent loss function is fully explained in Section \ref{sec:related works}. In the following, it is proved that under some conditions, our proposed loss function is Bayes consistent. 

\begin{theorem}
\label{Bayes consistency}
If each $\phi_{k}(y,f(\vec{x}))$ is a 
Bayes-consistent loss function, then $L_{s}=\sum_{k}{\lambda 
_{k}\phi_{k}(y,f(\vec{x}))}$ is also Bayes consistent. 
\end{theorem}
\paragraph{Note} In binary classifications, it is known that the decision function $f^*$ is Bayes consistent if the following equation holds \cite{bartlett2006convexity}. 
\begin{equation}\label{eq: Bayes consistency}
P(y=1|\vec{x})>\frac{1}{2}\rightarrow f^*(\vec{x})=1,\ \ \ P(y=1|\vec{x})<\frac{1}{2}\rightarrow f^{*}(\vec{x})=-1
\end{equation}

\begin{proof}
Let $y$ be $\lbrace +1,-1\rbrace $, by assuming $\eta 
=P(Y=1|X)$, the conditional expected risk under each individual loss 
function, $\lbrace \phi_{k}\rbrace _{k=1}^{m}$, is written as Equation \ref{eq:individual conditional risk}  \cite{masnadi2009design},
\begin{equation}\label{eq:individual conditional risk}
J_{k}=E\lbrack \phi(f(\vec{x}),Y)|X=\vec{x}\rbrack =\eta \phi_{k}(f(\vec{x}),1)+(1-\eta)\phi_{k}(f(\vec{x}),-1).
\end{equation}
\noindent
Each loss function $\lbrace \phi_{k}\rbrace _{k=1}^{m}$ is convex because it is Bayes-consistent. Thus, the optimal decision function $f_{k}^{*}$, can be obtained by setting derivatives of $J_k$ to zero, 
$$\frac{\partial 
J_{k}}{\partial f}=0 \Rightarrow \eta \frac{\partial \phi_{k}(f(\vec{x}),1)}{\partial 
f(\vec{x})}+(1-\eta )\frac{\partial \phi_{k}(f(\vec{x}),-1)}{\partial f(\vec{x})}=0\Longrightarrow f^*_k \text{ is obtained}$$
\noindent
The conditional expected risk for 
our proposed ensemble loss function is described as  
$$c_{L(\eta )}=E\lbrack L_s(f(\vec{x}),y)|\vec{x}\rbrack =\eta \sum_{k}{\lambda 
_{k}\phi_{k}(f(\vec{x}),1)}+(1-\eta )\sum_{k}{\lambda _{k}\phi_{k}(f(\vec{x}),-1)}$$
\noindent
Since $L_s(f(\vec{x}),y)$ is a linear combination of some convex functions, it is also convex. 
For example, by assuming $k=2$, the expected risk for the linear combination of two convex loss functions given an input $\vec{x}$ is written as   
\begin{equation*}
\begin{aligned}
J= c_{L(\eta )}=E\lbrack L(f(\vec{x}),y)|\vec{x}\rbrack =&\\
&\lambda _{1}\lbrack \eta 
\phi_{1}(f(\vec{x}),1)+(1-\eta )\phi_{1}(f(\vec{x}),-1)\rbrack +\\
&\lambda _{2}\lbrack \eta 
\phi_{2}(f(\vec{x}),1)+(1-\eta )\phi_{2}(f(\vec{x}),-1)\rbrack =\\
&\lambda _{1}J_{1}+\lambda 
_{2}J_{2}
\end{aligned}
\end{equation*}

\noindent
where $J_1, J_2$ have been shown in Equation \ref{eq:individual conditional risk}. The optimal decision function under our ensemble loss function is obtained by setting the derivative of $J$ to zero, 
$$\frac{\partial J}{\partial f}=0\Rightarrow\lambda _{1}\frac{\partial 
J_{1}}{\partial f(\vec{x})}+\lambda _{2}\frac{\partial J_{2}}{\partial f(\vec{x})}=0.
$$
 
\begin{figure}[h]
\centering
\begin{tikzpicture}[domain=-1:1,yscale=2,xscale=4,smooth,scale=0.6]
\draw[->][dashed] (0,0) -- (3,0.5) ;
\draw (0,0) node[anchor=north] {$f_1^*$} ;
\draw (3,0.5) node[anchor=north] {$f_2^*$} ;
\draw (2,4) node[anchor=south] {$J_1$} ;
\draw (1,4.5) node[anchor=south] {$J_2$} ;
\draw [color=blue,thick, variable=\x,domain=-0.5:2] plot ({\x},{\x*\x});
\draw [color=red,thick, domain=1:3.5,variable=\x] plot  ({\x},{(\x-3)*(\x-3)+0.5});

 \end{tikzpicture}
\caption{Expected Loss for Two  Bayes-consistent Loss Function \label{fig:beyes consistence}}
\end{figure}
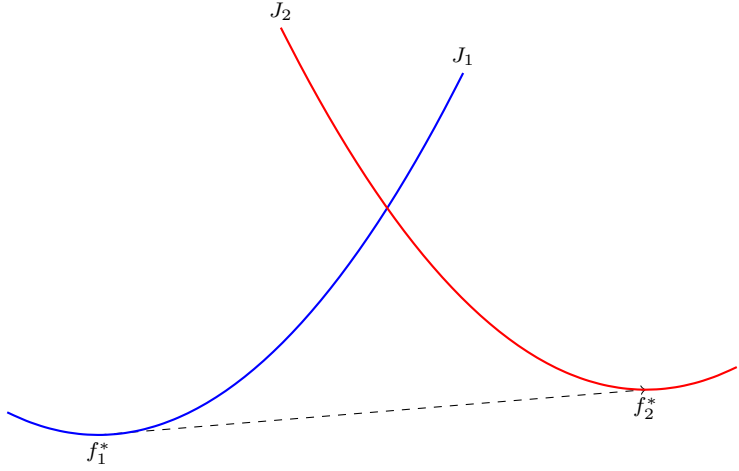

\noindent
Fig. \ref{fig:beyes consistence} shows an example of two convex functions $J_1$ and $J_2$ with the optimal point $f^*_1$ and $f^*_2$ respectively. By looking at Fig. \ref{fig:beyes consistence} the following equations are straightforward. 

\begin{equation}\label{eq:linear combination}
\begin{aligned}
\frac{\partial J_{1}}{\partial f(\vec{x})}\vert _{f_{1}^{*}}<0 \ \ \ \text{and} \ \ \  \frac{\partial J_{2}}{\partial f(\vec{x})}\vert _{f_{1}^{*}}<0 \Rightarrow \\
\frac{\partial J}{\partial f}\vert_{f_{1}^{*}}=\lambda_{1}\frac{\partial J_{1}}{\partial f(\vec{x})}\vert _{f_{1}^{*}}+\lambda_{2}\frac{\partial J_{2}}{\partial f(\vec{x})}\vert _{f_{1}^{*}}<0
\end{aligned}
\end{equation}
\begin{equation}\label{eq:linear combination2}
\begin{aligned}
\frac{\partial J_{1}}{\partial f(\vec{x})}\vert _{f_{2}^{*}}>0 \ \ \ \text{and} \ \ \  \frac{\partial J_{2}}{\partial f(\vec{x})}\vert _{f_{2}^{*}}>0 \Rightarrow \\
\frac{\partial J}{\partial f}\vert_{f_{2}^{*}}=\lambda_{1}\frac{\partial J_{1}}{\partial f(\vec{x})}\vert _{f_{2}^{*}}+\lambda_{2}\frac{\partial J_{2}}{\partial f(\vec{x})}\vert _{f_{2}^{*}}>0
\end{aligned}
\end{equation}
\noindent
Having considered Equation \ref{eq:linear combination} and \ref{eq:linear combination2}, we easily conclude that $f^{*}$ must 
lie between two points $f_{1}^{*}$ and $f_{2}^{*}$, and can be formulated as  $f^{*}=\alpha 
_{1}f_{1}^{*}+(1-\alpha _{1})f_{2}^{*}$. It means that $f^{*}$ is a linear combination of $ f^*_1, f^*_2$. Since for each $ f^*_1,f^*_2$ Equation \ref{eq: Bayes consistency} holds, it also holds for a linear combination of them. Therefore, $f^*$ is also Bayes consistent. This proof can easily be expanded for 
$k$ loss functions in the same way.  
\end{proof}

\subsection{Robustness}
\paragraph{\textbf{Robust loss function}} a distance-based loss function is said to be robust if there is a constant $k$ such that the loss function does not assign large values to samples with $e_i>k$ \cite{genton1998highly}. 

\smallskip
In the following, the mathematical explanation of the robustness is provided. Assuming a linear learner, the empirical risk is formulated as follows:-

\begin{equation*}
\begin{aligned}
J=R_{emp}=\sum_{i=1}^{n}\phi(y_i-\vec{x}_{i}\cdot \vec{w}).
\end{aligned}
\end{equation*}
 \noindent
where $\vec{x}_{i}\cdot \vec{w}$ denotes inner product of two vectors $\vec{x} \text{ and } \vec{w}$. To obtain the optimal value of $\vec{w}$, the following equation has to set to zero, 

\begin{equation*}
\begin{aligned}
\frac{\partial J}{\partial \vec{w}}=0\Rightarrow\sum_{i=1}^{n}\frac{\partial \phi(y_{i}-\vec{x}_{i=1} \cdot\vec{w})}{\partial \vec{w}}=\\
\sum_{i=1}^{n}{\phi'(y_{i}-\vec{x}_{i}\cdot \vec{w})\vec{x}_{i}}=0
\end{aligned}
\end{equation*}

\noindent
where $ e_i=y_{i}-\vec{x}_{i} \cdot\vec{w}$ . Given the weight function $w(e_{i})=\frac{\phi'(e_{i})}{e_{i}}$ , the above equation can be reformulated as follows:-

\begin{equation}
\sum_{i=1}^{n}{\phi '
(y_{i}-\vec{x}_{i}\cdot \vec{w})x_{i}}=\sum_{i=1}^{n}{w(e_{i})\vec{x}_{i}}.
\end{equation} 
\noindent
A loss function is robust if the following equation holds \cite{genton1998highly}.

\begin{equation}\label{eq:robustness}
\exists k, \ \ \  \forall \vec{x}_{i}, \ \lvert e_{i}\rvert >k\Longrightarrow w(e_{i})\to 0.
\end{equation}

\begin{theorem}\label{theorem:robustness}
If all individual loss functions are robust, our proposed ensemble loss function is also robust.
\end{theorem}

\begin{proof}

For the ensemble loss function, $w(e_i)$ is written as follows:- 

\begin{equation*}
w(e_i)=\sum_{k=1}^{m}{\lambda _{k}w_{k}(e_{i})}, \sum_{k=1}^m{\lambda 
_{k}}=1,\lambda _{k}>0 
\end{equation*}

\noindent
To prove Theorem \ref{theorem:robustness} we need to prove that $\exists k,\  \forall \vec{x}_{i}\mid \ \lvert e_{i}\rvert >k\Longrightarrow\sum_{k=1}^{m}{\lambda _{k}w_{k}(e_{i})}\to 0$. Since each individual 
loss function is robust then the following equation is straightforward. 
$$\exists k_{k},\ \ \forall \vec{x}_{i} ,\ \ \vert e_{i}\vert >k_{k}\Longrightarrow{w_k(e_{i})},\to 0
$$
where $w_k(.)$ is the weight function corresponding to $\phi_k(.)$. By assuming $k=\max \lbrace k_{k}\rbrace _{k=1}^{m}$, the proof of Equation \ref{eq:robustness} is straightforward.  
\end{proof}

\subsection{Training Phase using our Proposed Loss Function}

\noindent
We aim to minimize the empirical risk of our proposed loss function as follows 

\begin{equation}\label{eq:empirical}
\min _{\vec{w},\lambda} R_{emp}=\sum_{i=1}^{n} L_{s}(y_i - \vec{x}_{i}\cdot\vec{w})=\sum_{i=1}^{n}\sum_{k=1}^{m}\lambda 
_{k} \phi _{ik}(y_i- \vec{x}_{i}\cdot\vec{w})=\sum_{i=1}^{n}\sum_{k=1}^{m}\lambda 
_{k} \phi _{ik}(e_i)
\end{equation}

\noindent
where $L_{s}(.)$ denotes the value of ensemble loss function for $i$th sample, $\lambda 
_{k}$ denotes the weights associated with the $k$th loss function and $\phi _{ik}$ denotes the value of the $k$th loss function for the $i$th sample. We first omit $\lambda 
_{k}$ from Equation \ref{eq:empirical} and will later show that the weights associated with each 
loss function appear implicitly through HQ optimization. According to HQ 
optimization, Equation \ref{eq:empirical} is restated as follows

$$\min 
_{\vec{w},P}J=\sum_{i=1}^{n}{\sum_{k=1}^{m}{Q(P,e_{i})}+\psi 
_{k}(P)}$$

\noindent
where $\psi _{k}$ is the point-wise maximum of some convex functions and consequently
it is convex, $e_i=y-\vec{w}\cdot\vec{x}_i$ is the error associated with the $i$th sample and $P$ is the auxiliary variable. By substituting $Q(P,e_{i})$ with the multiplicative HQ function, 
the following equation is obtained.

\begin{equation}\label{eq:objective}
\min 
_{\vec{w},P}J(\vec{w},P)=\sum_{i=1}^{n}{\sum_{k=1}^{m}{\frac{1}{2}p_{ik}^{}(e_i)^{2}}+\psi 
_{k}(P)}=\sum_{i=1}^{n}{\sum_{k=1}^{m}{\frac{1}{2}p_{ik}^{}(y-\vec{w}\cdot\vec{x}_i)^{2}}+\psi 
_{k}(P)}
\end{equation}

\noindent
where $P$ is a matrix of $n$ rows and $m$ columns. And $p_{ik}$ denotes the element of $i$th row and $k$th column. To solve the above problem, we make use of HQ optimization algorithm which is iterative and alternating. Each iteration $(s)$ consists of two steps as follows \cite{he2014half}

\smallskip
\paragraph{\textbf{First}} by considering $\vec{w}=\vec{w}^{(s)}$ constant, we arrive at the following optimization problem with one variable $P$

\begin{equation}\label{eq:p}
\begin{aligned}
P^{(s+1)}=\arg\min_{P}\sum_{i=1}^{n}\sum_{k=1}^{m}\frac{1}{2}p_{ik}(y_{i}-\vec{w}^{(s)}\cdot \vec{x}_{i})^{2}+\psi_{k}(P). \\
\end{aligned}
\end{equation}

\noindent
The value of $P^{(s+1)}$ is calculated using $\delta(.)$ function which is pre-calculated for some special loss functions $\phi(.)$ and their corresponding conjugate functions $\psi(.)$ which are listed in Fig. \ref{HQ}. The optimal value of $p_{ik}$ is obtained by  $\delta_{k}(e_{i})$.

\smallskip
\paragraph{\textbf{Second}} given the obtained value of $P^{(s+1)}$, we optimize over $\vec{w}$ as follows

\begin{equation}\label{eq:w(1)}
\begin{aligned}
\vec{w}^{(s+1)}=\arg\min_{\vec{w}}j=&\arg\min_{\vec{w}}\sum_{i=1}^{n}\sum_{k=1}^{m}\frac{1}{2}p_{ik}^{(s+1)}(y_{i}-\vec{w})\cdot \vec{x}_{i})^{2} \\
\frac{\partial j}{\partial \vec{w}}=0\Rightarrow \vec{w}^{(s+1)}=&\left[\underbrace{ \sum_{i=1}^{n}\vec{x}_{i}\times \vec{x}_{i}^T\times\sum_{k=1}^{m}{p_{ik}^{(s+1)}}}_{*}\right]^{-1}\sum_{i=1}^{n}\vec{x}_{i}\sum_{k=1}^{m}p_{ik}^{(s+1)}y_{i}
\end{aligned}
\end{equation}

\noindent
where $P^{(s)}$ and $\vec{w}^{(s)}$ are respectively the optimal values for matrix $P$ and vector $\vec{w}$ in the $s$th iteration. The matrix represented by $*$ is generally Positive Semi-Definite (PSD) and might not be invertible. To make it Positive Definite (PD) we add it a diagonal matrix with strictly positive diagonal entries \cite{meyer2000matrix}. The refinement of Equation \ref{eq:w(1)} is as follows

\begin{equation}\label{eq:w}
\begin{aligned}
\vec{w}^{(s+1)}=&\left[\underbrace{ \sum_{i=1}^{n}\vec{x}_{i}\times \vec{x}_{i}^T\times\sum_{k=1}^{m}{p_{ik}^{(s+1)}}}_{*}+\alpha \times I\right]^{-1}\sum_{i=1}^{n}\vec{x}_{i}\sum_{k=1}^{m}p_{ik}^{(s+1)}y_{i}
\end{aligned}
\end{equation}

\noindent
where $\alpha\in R$ is a very small positive number and $I$ is the identity matrix. The Equation \ref{eq:p} and \ref{eq:w} are performed iteratively in an alternating manner until convergence.

\medskip
$P^{(s+1)}$ denotes 
matrix $P$  in $s$th iteration. $p_{ik}^{(s)}$ is the element of the $i$th row and the 
$k$th column of matrix $P$, so, $p_{ik}$ can be considered as the 
weight of the $i$th sample associated with the $k$th loss 
function. Thus, we can calculate the total weight of the $k$th loss function by $\lambda_{k}=\sum_{i=1}^{n}{p_{ik}^{}}$.

\begin{proposition}
 The sequence $\lbrace 
J(P^{(s)},\vec{w}^{(s)}),s=1,2,\ldots \rbrace $ generated by Algorithm \ref{algorithm1} converges.
\end{proposition}

\begin{proof}
According to Equation \ref{eq:p} and the properties of the minimizer function $\delta(.)$ we have 
$$J(\vec{w}^{(s)},P^{(s+1)})\leq J(\vec{w}^{(s)},P^{(s)})$$
for a fixed $\vec{w}^{(s)}$. And by Equation \ref{eq:w}, for a fixed $p^{(s+1)}$ we have
$$J(\vec{w}^{(s+1)},P^{(s+1)})\leq J(\vec{w}^{(s)},P^{(s+1)}).$$
Summing up the above equations gives us:

$$J(\vec{w}^{(s+1)},P^{(s+1)})\leq J(\vec{w}^{(s)},P^{(s+1)})\leq J(\vec{w}^{(s)},P^{(s)}).$$
 Since the cost function $J(P,\vec{w})$ is below bounded, the sequence $$ \left\lbrace \ldots (\vec{w}^{(s+1)},P^{(s+1)})\leq J(\vec{w}^{(s)},P^{(s+1)})\leq J(\vec{w}^{(s)},P^{(s)})\ldots \right\rbrace$$
decrease and finally converges when  $s\rightarrow \infty$ \cite{he2014half}. 
\end{proof}

\begin{algorithm}
\caption{Our Proposed Ensemble Loss Function using HQ Optimization Algorithm (RELF)}\label{algorithm1}
\textbf{INPUT:}\\
 Streaming samples $\lbrace \vec{x}_{i},y_{i}\rbrace 
_{i=1}^{n}$ \\
 Base loss functions $\lbrace \phi 
_{i}(\vec{x}_{i},y_{i})\rbrace _{i=1}^{m}$\\
\textbf{OUTPUT:} \\
Parameter ${\vec{w}, p}$ in Equation \ref{eq:objective} and $\lbrace\lambda_i\rbrace_{i=1}^m$ (weights associated with each base loss function $\phi_i(.)$)
\begin{algorithmic}[1]

\State Initiate $L_{s}$ using $\lbrace \phi 
_{i}(\vec{x}_{i}.y_{i})\rbrace _{i=1}^{m}$ and random $\lambda_i$
\State Initialized parameters $\vec{w}\sim 
N(0.\Sigma )$ 

\While {\text{ until convergence }} 
\State Compute $\lbrace \phi _{ik}\rbrace _{k=1}^{m}$ for $(\vec{x}_{i}.y_{i})$
\State Update $P$ by $p_{ik}=\delta (\vec{x}_{i},y_{i})$ according to Equation \ref{eq:p}
\State
Update learning parameter $\vec{w}$ according to 
\ref{eq:w} 
\State
$s \gets s+1$
\EndWhile 
\State
$\lambda_j=\sum_{i=1}^n p_{ij}$
\State
\Return $\lbrace P^{(s)}, \vec{w}^{(s)},{\lambda_i}_{i=1}^n \rbrace$ 
\end{algorithmic}
\end{algorithm}

\subsection{Toy Example}

To clarify our proposed approach, we carry out an experiment on small synthetic data. $x$ as input are generated from 
$[-20:20]$ by step size 0.5. Labels lie in a line defined by $y=2\times x+z$, where $z$ denotes noise. We ensemble two loss functions, Welsch and $l_{1}\_l_{2}$, which are selected from functions listed in Fig. \ref{HQ}. We conduct our experiments in presence of two kinds of noises: (1) Zero-mean Gaussian noise and (2) Outliers.

\smallskip
The weights associated with each base loss function and $w$ are iteratively updated according to Equation \ref{eq:p} and \ref{eq:w} respectively. These weights are presented in Table \ref{table:toy}. 
The weights associated with Welsch loss function decreases in presence of outliers while it increases for the $l_{1}\_l_{2}$ function. Welsch and the $l_{1}\_l_{2}$ functions are 
plotted in Fig. \ref{fig: toy}. 

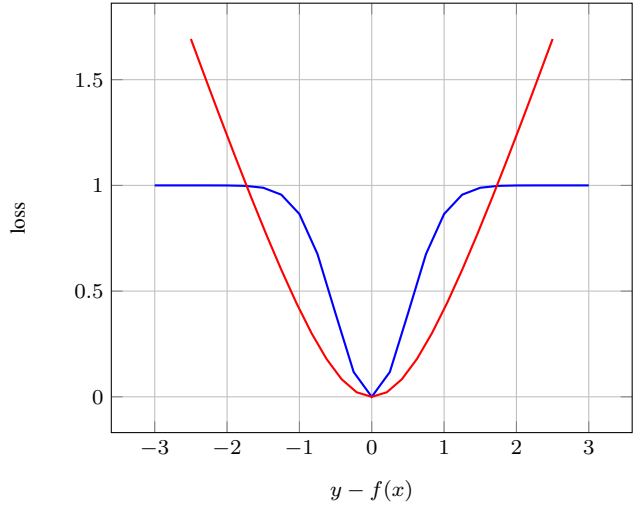
\begin{figure}[h]
\centering
\begin{tikzpicture}
\begin{axis}[grid,xlabel={$y-f(x)$},ylabel=loss]
\addplot [color=blue,thick, domain=-3:3]{1-exp(-2*x^2))};
\addplot [color=red,thick, domain=-2.5:2.5] plot {sqrt(1+x^2)-1};
\end{axis}
\end{tikzpicture}
\caption{Welsch function (red line), $l_{1}\_l_{2}$ function(red line) \label{fig: toy}}
\end{figure}

\begin{table}[h]
\caption{Results for syntactic data \label{table:toy}}
\centering
\begin{tabular}{l l l}
\hline
$\lambda _{i}$ & Gaussian noise & Outliers \\
\hline
$\lambda _{Welsch}$\\
 & 0.1546 & 0.3578 \\
$\lambda _{l_{1}\_l_{2}}$\\
 & 0.8453 & 0.6422 \\
$w$ & 1.9975 & 1.9969 \\
\hline
\end{tabular}
\end{table}

As shown in Fig. \ref{fig: toy}, Welsch is a bounded function and the $l_{1}\_l_{2}$ is unbounded. Therefore, the $l_{1}\_l_{2}$ assigns a larger value to a sample with large  error in comparison with Welsch. Therefore, to decrease the influence of outlier on the predicted function, the value of $\lambda_i$ for 
$l_{1}\_l_{2}$ is less than the corresponding value for Welsch 
function in presence of outliers.   

We also combine three, four and five base loss functions in Fig. \ref{HQ} and the weights associated with each loss function are reported in Table \ref{table: toy2}. Results show that among these five loss functions, Fair loss function contributes the most to form the ensemble loss function. Also, the predicted $w$ has been above value $1.99$ for all experiments.

\begin{table}[h]
    \centering
  \begin{tabular}{c l l l l l l }
    \hline
         Number of base loss functions & kind of noise& \ Welsch & $l_{1}\_l_{2}$ &  Huber &  Fair & Log-cosh \\
         \hline
        3 &Gaussian& 0.26 & 0.58 & 0.14&--&--\\
          3 &Outlier& 0.28 & 0.62 & 0.09&--&--\\
         \hline
        4 &Gaussian& 0.15 &	0.36&	0.07&	0.38&--\\
       4 &Outlier &0.16&	0.37&	0.05&	0.41&--\\
       \hline
      5 &Gaussian& 0.13 &	0.28&	0.05&	0.30& 0.21\\
      5 &Outlier &0.11&	0.27&	0.04&	0.33&0.23\\
      \hline
        
    \end{tabular}
    \caption{Weights associated with each base loss function}
    \label{table: toy2}
\end{table}

\section{Experiments}\label{sec: experiments}

In this section, we evaluate how our proposed loss function works in regression problems. To make our proposed ensemble loss function stronger, we have chosen base loss functions which are diversified by different behaviour against outliers. 
Welsch function is bounded and assigns a small value to samples with large errors. Huber loss function 
penalizes samples linearly with respect to $e_i=y-\vec{x}_{i} \cdot \vec{w}$, while $l_{1}\_l_{2}$  function highly penalizes samples with large $e_i$. Therefore, these three loss functions which are completely diversified by the behaviours against outliers have been chosen as the base functions of our proposed ensemble loss function.  

\medskip
We have conducted our experiments on some different benchmark datasets which are briefly described in Table \ref{table: normal situation}. Mean-Absolute Error (MAE) is utilized to compare 
the results which is calculated as $\frac{1}{N}\sum_{i=1}^N |y_i-\hat{y}_i|$ where $N$ is the number of test samples and $y_i, \hat{y}_i$ are the true and estimated label respectively. In all experiments, we have used $10$ fold cross validation for model selection. It means that the original dataset is partitioned into 10 disjoint subsets. Afterwards, we have run $10$ iterations in each $9$ subsets have been used for training and the remaining one for testing. Every subset has been exactly used once for testing. And the best model parameters has been selected.

\medskip
The data have been initially normalized. 
Table. \ref{table: normal situation} shows the experimental results on some benchmark datasets in natural situation (without adding outliers to samples). It shows that the performance 
of all regressors are somehow the same in all datasets; however, 
RELF achieves the best results.

\begin{table}[h!]
\caption{The Mean Absolute Error (MAE)  comparison results on several datasets. The best result for each dataset is presented in bold. \label{table: normal situation}}
\centering
\begin{tabular}[\textwidth]{l l l l l l l}
\hline
Dataset name & \# of samples & \# of features & Lasso & LARS & SVR & 
RELF \\
\hline
Airfoil self-noise & 1503 & 6 & 3.5076 & 3.5147 & 3.7954 & \textbf{
3.4937} \\
Energy Efficient Dataset & 768 & 8 & 1.9741 & 1.9789 & \bf 1.4021 
& 1.9476 \\
Red Wine Quality Dataset & 4898 & 12 & 2.6605 & 2.6637 & 2.7630 & 
\textbf{2.6561} \\
White Wine Quality Dataset & 4898 & 12 & 1.9741 & 1.9789 & 1.4021 & 
\textbf{1.9733} \\
Abalone Dataset & 4177 & 8 & 2.0345 & 1.8693 & 1.5921 & \textbf{1.5549
} \\
Bodyfat\_Dataset & 252 & 13 & 0.91823 & 0.84867 & 0.6599 & \textbf{
0.52465} \\
Building\_Dataset & 4208 & 14 & 1.017 & 1.0183 & 2.4750 & \textbf{
1.0024} \\
Engine\_Dataset & 1199 & 2 & 1.3102 & 1.4113 & 0.9833 & \textbf{0.88694
} \\
Vinyl\_Dataset & 506 & 13 & 1.1501 & 1.0598 & 0.9854 & \textbf{0.23952
} \\

Simplefit\_Dataset & 94 & 1 & 0.63349 & 0.51819 & 0.4836 & \textbf{
0.38523} \\
\hline
\end{tabular}
\end{table}
 
\subsection{Discussion about the Convergence of HQ Optimization Method}

In this section, the convergence of HQ optimization method in our algorithm is experimentally studied. Table \ref{table: convergence} shows values of cost function, $J(\vec{w},P)$, in $30$ successive  
iterations. We define a \textit{decrease ratio} in Equation \ref{eq: decrease} which  shows how much cost function being reduced in iteration 10 to 30. This ratio has been calculated for each dataset separately and shown in Table \ref{table: convergence}. The results shows that the HQ optimization method quickly converges within 30 
iterations in all cases. Moreover,   $J(\vec{w}^{(10)},P^{(10)})$ is a good approximation of the optimal point. We have also provided the CPU time in seconds for each dataset in Table \ref{table: convergence}.  

\begin{equation}\label{eq: decrease}
decrease \ \ 
ratio=\frac{J(\vec{w}^{(1)},P^{(1)})-J(\vec{w}^{(10)},P^{(10)})}{J(\vec{w}^{(1)},P^{(1)})-J(\vec{w}^{(30)},P^{(30)})}
\end{equation}

\begin{table}[h!]
\centering
\caption{Values of Cost Function, $J(\vec{w},p)$ in several successive Iterations \label{table: convergence}}
\begin{tabular}[\textwidth]{p{2.5cm} p{1 cm} p{1cm} p{0.5cm} p{1cm} p{0.5cm} p{1cm} p{1cm} p{1cm} p{1cm}}
\hline
Dataset name & 1 & 2 & \ldots & 10 & \ldots & 29 & 30 & Decrease Ratio &  CPU time 
\\
\hline

Airfoil self-noise & 617.9089 & 615.9067 &. . & 615.2434 &. . & 615.2074 & 
615.2070 & 0.986528 &0.3125\\

Energy Efficient & 298.9774 & 294.7328 &. . & 291.9451 & . .& 291.9144 & 291.8603 & 0.988085 & 0.6094\\

Red Wine Quality & 658.9971 & 657.8419 &. . & 657.5159 & . .& 657.4989 & 657.4986 & 0.988455 &0.2656\\

White Wine Quality & 324.2221 & 291.8042 &. . & 291.8410 & . .& 
291.6763 & 291.6837 & 0.995166& 1\\

Abalone& 1.5604 & 1.5490 & . .& 1.5362 &. . & 1.5331 & 1.5330 & 
0.883212& 0.6250\\

Bodyfat& 98.6765 & 98.4771 & . .& 98.4393 &. . & 98.4389 & 98.4389 & 
0.998316 &0.0469\\

Building& 3.3932 & 3.3650 &. . & 3.3434 & . .& 3.3400 & 3.3399 & 
0.934334 &    0.8438
\\

Engine & 1.6640 & 1.6026 & . .& 1.5500 &. . & 1.5397 & 1.5394 & 
0.914928 & 0.5313\\

Vinyl & 2.5582 & 2.5574 &. . & 2.5573 &. . & 2.5573 & 2.5573 & 1 &1.9063 \\

Simplefit & 44.2226 & 43.4881 & . .& 43.1884 &. . & 43.1652 & 43.1649 
& 0.977782 &0.0156\\
\hline
\end{tabular}
\end{table}

\subsection{Discussion about Robustness}

In this section, we investigate the robustness of our proposed loss function through the experiment on some benchmark datasets which are provided in Table \ref{table: normal situation}. $10$ fold cross validation has been used to tune the hyper-parameters. To 
study the robustness, we add outliers to training and validation samples. We conduct experiments on data which are corrupted with various level of 
outlier including $10\%$ and $30\%$. $30\%$ outliers which means that we randomly select $30\%$ of samples and add outliers to their labels. Table \ref{table: 10 percent}
and Table \ref{table: 30 percent} show the increase ratio of MAE in the face of outliers in comparison of natural situation (without outliers). The best results are presented with bold. 
Lasso, LARS, and SVR are three well-known regressions which have been used to make comparison.

\begin{table}[h]
\centering
\caption{Increase Ratio of MAE in the Face of 
Outliers (10\%). The best result for each dataset is presented in bold. \label{table: 10 percent}}
\begin{tabular}[\textwidth]{l l l l l}
\hline
Dataset Name & Lasso & LARS & SVR & RELF \\
\hline
Airfoil self-noise & 1.0127 & 1.0191 & 1.0227 & \textbf{1.0016} \\
Energy Efficient Dataset & 1.4025 & 1.3641 & 1.0971 & \textbf{1.0082} 
\\
Red Wine Quality Dataset & 1.0683 & 1.0544 & 1.0235 & \textbf{1.0139} 
\\
White Wine Quality Dataset & 1.2941 & 1.2897 & 1.0971 & \textbf{1.0602
} \\
Abalone Dataset & 1.0264 & 1.1246 & 1.0183 & \textbf{1.0077} \\
Bodyfat\_Dataset & 1.5125 & 1.5279 & \textbf{1.0242} & 1.0692 \\
Building\_Dataset & 1.0016 & 1.0014 & 1.0143 & \textbf{1.0004} \\
Engine\_Dataset & 1.1374 & 1.0242 & 1.0569 & \textbf{1.0089} \\
Vinyl\_Dataset & 1.2345 & 1.148 & 1.1987 & \textbf{1.0978} \\

Simplefit\_Dataset & 1.3122 & 1.8628 & \textbf{1.0023} & 1.2359 \\
\hline
\end{tabular}
\end{table}

\begin{table}[h]
\centering
\caption{Increase ratio of MAE in the face of 
outliers (30\%). The best result for each dataset is presented in bold. \label{table: 30 percent}}
\begin{tabular}{l l l l l}
\hline
Dataset Name & Lasso & LARS & SVR & RELF \\
\hline
Airfoil self-noise & 5.9316 & 6.6159 & 1.0213 & \textbf{1.0178} \\
Energy Efficient Dataset & 2.544 & 2.0297 & 1.1555 & \textbf{1.0016} 
\\
Red Wine Quality Dataset & 1.408 & 2.0033 & 1.0270 & \textbf{1.0024} 
\\
White Wine Quality Dataset & 2.1278 & 2.0297 & 1.1555 & \textbf{1.0219
} \\
Abalone Dataset & 2.0264 & 3.0233 & 1.0338 & \textbf{1.0318} \\
Bodyfat\_Dataset & 2.3491 & 1.1265 & 1.0338 & \textbf{1.2882} \\
Building\_Dataset & 1.985 & 1.1438 & 1.0151 & \textbf{1.0022} \\
Engine\_Dataset & 2.1156 & 2.0233 & 1.3665 & \textbf{1.0061} \\
Vinyl\_Dataset & 3.5624 & 2.2455 & 1.2785 & \textbf{1.0651} \\

Simplefit\_Dataset & 8.674 & 8.3125 & \textbf{1.4424} & 1.5212 \\
\hline
\end{tabular}
\end{table}

Table \ref{table: 10 percent} lists the ratio of MAE value for data with $10\%$ outliers to MAE value for original data. Except two datasets RELF has been least influenced by outliers among other regressors. Table \ref{table: 30 percent} presents the numerical comparison results for the four mentioned models in the presence of $30\%$ outliers. RELF was least influenced by outliers and SVR got the better results in comparison with LASSO and LARS.

\subsection{Comparison with State-of-the-Art Ensemble Regressors}
We also investigate the effectiveness of our proposed method in comparison with several promising ensemble regressors through experiments on some benchmark datasets. The datasets have been selected from LIBSVM data page\footnote{\url{https://www.csie.ntu.edu.tw/~cjlin/libsvmtools/datasets/}}. 
We compare RELF with four ensemble regression methods: Locally Weighted Linear Regression (LWLR) \cite{zhang2016online}, Locally Linear Ensemble for Regression (LLER) \cite{kang2018locally}, Artificial Neural Network (ANN) \cite{rodriguez2015machine} and Random Forest (RF) \cite{rodriguez2015machine}. The first two models are based on combining local experts and the next two models are nonlinear. In the following, the implementation settings for each method is fully provided. 

\smallskip
The LWLR which is based on local expert method aims to build a local linear regressor for each test sample based on its neighbours. The training of each linear model is such that it assigns higher weights to those training samples which are closer to the given test sample. The weights are calculated according to Gaussian kernel which is fully explained in \cite{schaal2002scalable}. 

\smallskip
For the ANN, the number of hidden nodes are selected from $\lbrace 1, 2, 3, 5, 10, 15, 20, 25, 30\rbrace$ set. If the amount of loss for validation batch fails to decrease during six successive epochs, the training phase ends.

\smallskip
To implement RF, we set the number of trees to 100 and bootstrap sample size to 80 percent of training samples for each tree. We set the minimum size of leaf nodes to \{0.1, 0.2, 0.5, 1, 2, 5\} percent of training samples.

\medskip
In all experiments, we use $10$ fold cross validation for parameter tuning, and the data are initially normalized into the range [-1, 1]. Root-Mean-Square-Error (RMSE) is utilized to compare the results which is calculated as $\sqrt[]{\frac{1}{N}\sum_{i=1}^N (y_i-\hat{y}_i)^2}$ where $N$ is the number of test samples and $y_i, \hat{y}_i$ are the true and estimated label respectively. Table \ref{table:datasets} provides some extra information about each dataset.

\begin{table}[H]
\caption{Datasets' description \label{table:datasets}}
\centering
\begin{tabular}{l l l l}
\hline
\bf Dataset & \bf Source & \bf Number of Data& \bf Number of Feature \\
\hline
BodyFat & Source: StatLib / bodyfat& 
 252& 
14\\
Abalone& Source: UCI / Abalone& 
 4,177& 
8 \\
Cadata& Source: StatLib / houses.zip&
 20,640&
 8\\
Cpusmall&Source: Delve / comp-activ&   
 8,192&
 12\\
Housing&Source: UCI / Housing (Boston)&
506&
 13 \\
Mg & Source: [GWF01a]&
1,385&
6\\
\hline

\end{tabular}
\end{table}

The five mentioned models have been numerically compared on 6 datasets and the results are plotted by bar charts in Figure \ref{fig:ensemble comparison}. NN and RF models are nonlinear while LWLR and LLER are locally linear. The RELF yields the lowest RMSE value for almost all datasets except for Abalone and Cpusmall that in which NN and RF get the lowest RMSE value.  

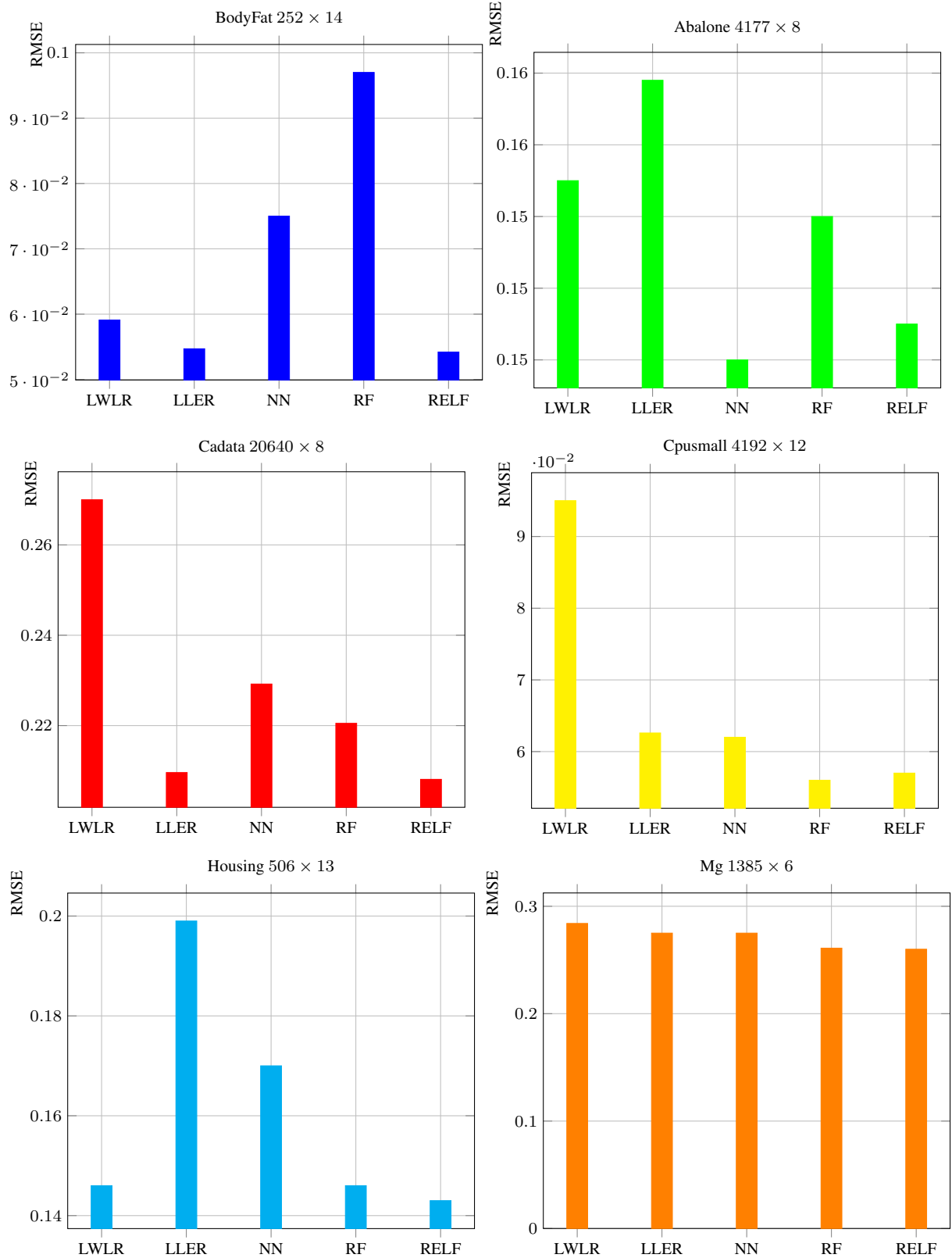
\begin{figure}[htpb]

\begin{subfigure}{0.5\textwidth}\begin{center}
\begin{tikzpicture}
\begin{axis}
	 [title  =  BodyFat $252 \times 14$,
     xmajorgrids = true,
     ymajorgrids = true,
     ybar,
     symbolic x coords={LWLR,LLER,NN,RF,RELF},
      y label style={at={(current axis.above origin)}},
    ylabel=RMSE,
	bar width=10pt,
    ]
\addplot [blue, fill=blue] coordinates {
(LWLR,0.0591) 
(LLER,0.0547)
(RF, 0.097)
(NN, 0.075)
(RELF,0.0542)
};
\end{axis}
\end{tikzpicture}
\end{center}
\end{subfigure}%
\begin{subfigure}{.5\textwidth}
\begin{center}
\begin{tikzpicture}
\begin{axis}
[
	 title  =  Abalone $4177 \times 8$,
     ybar,
     xmajorgrids = true,
     ymajorgrids = true,
     symbolic x coords={LWLR,LLER,NN,RF,RELF},
    y label style={at={(current axis.above origin)},right=0.1mm},
    ylabel=RMSE,
	bar width=10pt,
    ]
\addplot [green, fill=green] coordinates {
(LWLR,0.155) 
(LLER,0.1578)
(RF, 0.154)
(NN, 0.150)
(RELF,0.151)
};
\end{axis}
\end{tikzpicture}
\end{center}
\end{subfigure}

\medskip
\begin{subfigure}{.5\textwidth}
\begin{center}
\begin{tikzpicture}
\begin{axis}
[
	 title  =  Cadata $ 20640 \times 8$,
     ybar,
     xmajorgrids = true,
     ymajorgrids = true,
     symbolic x coords={LWLR,LLER,NN,RF,RELF},
     y label style={at={(current axis.above origin)},below=2mm},
    ylabel=RMSE,
	bar width=10pt,
    ]
\addplot [red, fill=red] coordinates {
(LWLR,0.2700) 
(LLER,0.2096)
(RF, 0.2205)
(NN, 0.2292)
(RELF,0.2081)
};
\end{axis}
\end{tikzpicture}
\end{center}
\end{subfigure}
\begin{subfigure}{.5\textwidth}
\begin{center}
\begin{tikzpicture}
\begin{axis}
[
	 title  =  Cpusmall $4192 \times 12$,
     ybar,
     xmajorgrids = true,
     ymajorgrids = true,
     symbolic x coords={LWLR,LLER,NN,RF,RELF},
     y label style={at={(current axis.above origin)},below=2 mm},
    ylabel=RMSE,
	bar width=10pt,
    ]
\addplot [yellow, fill=yellow] coordinates {
(LWLR,0.095) 
(LLER,0.0626)
(RF, 0.056)
(NN, 0.062)
(RELF,0.057)
};
\end{axis}
\end{tikzpicture}
\end{center}
\end{subfigure}

\medskip
\begin{subfigure}{.5\textwidth}
\begin{center}
\begin{tikzpicture}
\begin{axis}
[
	 title  =  Housing $ 506 \times 13$,
     ybar,
     symbolic x coords={LWLR,LLER,NN,RF,RELF},
     xmajorgrids = true,
     ymajorgrids = true,
y label style={at=(current axis.above origin),anchor=south},
    ylabel=RMSE,
	bar width=10pt,
    ]
\addplot [cyan, fill=cyan] coordinates {
(LWLR,0.146) 
(LLER,0.199)
(RF, 0.146)
(NN, 0.170)
(RELF,0.143)
};
\end{axis}
\end{tikzpicture}
\end{center}
\end{subfigure}
\begin{subfigure}{.5\textwidth}
\begin{center}
\begin{tikzpicture}
\begin{axis}
[
	 title  =  Mg $  1385 \times 6$,
     ybar,
     xmajorgrids = true,
     ymajorgrids = true,
     symbolic x coords={LWLR,LLER,NN,RF,RELF},
     y label style={at=(current axis.above origin),anchor=south},
    ylabel=RMSE,
    ymin=0, 
	bar width=10pt,
    ]
\addplot [orange, fill=orange] coordinates {
(LWLR,0.284) 
(LLER,0.275)
(RF, 0.261)
(NN, 0.275)
(RELF,0.260)
};
\end{axis}
\end{tikzpicture}
\end{center}
\end{subfigure}
\caption{RMSE comparison results on benchmark datasets.  \label{fig:ensemble comparison}}
\end{figure}

\section{Conclusion and Future Works}\label{sec: conclusion}

Inspired by ensemble methods, in this paper we proposed a regression extended with an ensemble loss function which is called RELF. We focused on regression tasks and it is considered as an initial attempt to explore the use of an ensemble loss functions. In this work, we utilized 
Half-Quadratic programming to find the optimal weight associated with each loss function. 

\medskip
RELF has been implemented and evaluated in noisy environments which showed a significant improvement with respect to outliers in comparison with simple regressors. We have also investigated the performance of RELF in comparison with some promising ensemble regressors through experiments on several benchmark datasets. Moreover, according to our results, the HQ 
optimization method has quickly converged. 

\medskip
As a future work, we aim to use an evolving solution that all the loss functions are applied and ranked on non-dominance. All those that are not dominated become the leaders. It would converge faster than individually applied loss functions. As another future work, we plan to make our proposed ensemble loss function sparse by adding a regularization term. In addition, although the focus of this paper is on regression, The concept of ensemble loss function can also be applied to other classification based applications. Another possible extension is making the ensemble loss function sparse by adding a regularization term. 

\bibliographystyle{plain}
\bibliography{main}

\end{document}